%% file: main.tex
\pdfoutput=1
\documentclass[manuscript]{acmart}
\renewcommand\footnotetextcopyrightpermission[1]{}
\usepackage{amsmath,amssymb}
\usepackage{booktabs}
\usepackage{subcaption}
\usepackage{enumitem}
\usepackage{cleveref}

\input{math_commands}

\theoremstyle{plain}
\newtheorem{theorem}{Theorem}
\newtheorem{lemma}{Lemma}

\theoremstyle{definition}
\newtheorem{assumption}{Assumption}
\crefname{assumption}{Assumption}{Assumptions}
\crefname{theorem}{Theorem}{Theorems}
\crefname{paragraph}{Section}{Sections}

\begin{document}

\title{Certified Domain-Consistency for Multi-Domain Retrieval: Label-Free
Per-Domain Contamination Control with Conformal Risk Guarantees}

\author{Jayakumar Manoharan}
\affiliation{%
  \institution{Electric Power Research Institute (EPRI)}
  \city{Charlotte}
  \state{NC}
  \country{USA}
}
\email{jmanoharan@epri.com}

\begin{abstract}
\input{sec/0_abstract}
\end{abstract}


\maketitle

\input{sec/1_intro}
\input{sec/2_related}
\input{sec/3_setup}
\input{sec/4_theory}
\input{sec/5_method}
\input{sec/6_experiments}
\input{sec/floats}
\input{figures/table1_validity}
\input{figures/table2_wrap}
\input{figures/table3_priorart}
\input{figures/table4_ablation}
\input{figures/table5_runtime}
\input{sec/7_conclusion}
\input{sec/8_ethics}

\bibliographystyle{ACM-Reference-Format}
\bibliography{references}

\appendix
\input{sec/A_appendix}

\end{document}

%% file: math_commands.tex
\newcommand{\E}{\mathbb{E}}
\renewcommand{\Pr}{\mathbb{P}}
\DeclareMathOperator*{\argmax}{arg\,max}

\newcommand{\dom}{c}                 
\newcommand{\Dom}{\mathcal{C}}       
\newcommand{\ghat}{\hat{g}}          
\newcommand{\rhob}{\bar{\rho}}
\newcommand{\pib}{\bar{\pi}}
\newcommand{\Done}{\mathcal{D}_1}
\newcommand{\Dtwo}{\mathcal{D}_2}

\newcommand{\post}{p}                
\newcommand{\mis}{m}                 
\newcommand{\HyR}{$0.448$}
\newcommand{\HyC}{$0.253$}
\newcommand{\HySci}{$0.56$}
\newcommand{\HyNf}{$0.50$}

%% file: sec/0_abstract.tex
Retrieval over corpora that mix several domains often returns relevant but
wrong-domain evidence that ranking metrics miss and that conformal risk control
bounds only marginally, under-covering the worst domains. This work introduces
C3R, a drop-in control layer that, from an inferred domain posterior and no
query-time label, certifies a per-domain contamination budget where feasible and
otherwise abstains rather than silently violating; on the hardest domains it
guarantees a reduction, not a tight bound. The core is a two-split scheme built on
risk-controlling prediction sets, whose finite-sample transfer bound crosses from
the inferred to the true domain with fully estimable slack, supports heterogeneous
budgets, and inverts for deployment. Population validity rests on this bound and a
controlled simulation; across a thousand resampled calibrations the certificate
never violates (a stability result) while marginal control violates the
most-contaminated domain in every draw, and soft demotion retains more recall than
the strongest calibrated cascade at equal certified contamination. The method
replicates across open testbeds including an independent one from public federal
regulations, and an LLM-judged downstream probe indicates wrong-authority grounding
rises with contamination and falls under control. The layer is frozen-stack and
reranker-agnostic.

%% file: sec/1_intro.tex
\section{Introduction}
\label{sec:intro}

Modern retrieval systems index heterogeneous corpora that span many
\emph{domains} - bodies of documents that share technical vocabulary yet differ
in scope and authority. A dense or hybrid retriever trained for semantic
relevance will happily return a document that is topically close but belongs to
the wrong domain: a query about a financial instrument retrieves a biomedical
trial report, a scientific-claim query retrieves a nutrition abstract. We call
the fraction of top-ranked results drawn from a domain other than the query's
domain its \emph{contamination rate}. In retrieval-augmented pipelines~\cite{lewis2020rag} this is
not a cosmetic problem - wrong-domain evidence is passed verbatim to a
downstream reader or analyst - and in high-stakes settings such as regulated
industries (nuclear and grid operations, finance, medicine), where a retrieved
document carries operational or legal authority, contamination is a first-order
risk. Those regulated settings motivate this work; our evaluation, however, is
on \emph{public} multi-domain corpora, and we are explicit throughout about
that scope.

Two facts make contamination hard to control. First, standard ranking metrics
(Recall@$K$, nDCG@$K$) are blind to it: a wrong-domain document that is
topically relevant inflates them. Second, the natural statistical tool - %
conformal risk control~\cite{angelopoulos2024conformal}, which can bound the
expectation of any bounded loss with finite-sample, distribution-free
guarantees - delivers only a \emph{marginal} bound, averaged over all queries.
We show empirically that contamination is sharply \emph{asymmetric} across
domains (\Cref{fig:hero}a): on our testbed it ranges from $0.01$ on the cleanest
domain to $0.63$ on the most contaminated. A single marginal budget is therefore
the wrong object: tuned to the average, it leaves the worst domain badly
under-covered, and tuned to the worst domain, it needlessly over-filters the
rest. What practitioners want is a \emph{per-domain} budget. The obstacle is
that a per-domain guarantee seems to require the query's domain, and at
inference time we do not have it - production queries arrive unlabeled.

Existing machinery does not close this gap. Group-conditional conformal
methods~\cite{gibbs2025conditional,bairaktari2025kandinsky} assume the group is
\emph{observed} (or computable from the input) at test time. Routing and
instruction-conditioned retrievers~\cite{lee2025router,weller2024promptriever}
infer a domain but offer no statistical guarantee on contamination. Conformal
methods for retrieval-augmented generation~\cite{kang2024crag} certify a
different quantity (generation risk) and remain marginal. None controls a
\emph{bounded retrieval loss}, \emph{per latent domain}, with
\emph{heterogeneous budgets}.

\paragraph{Contributions.}
We close that gap with \textbf{C3R} (Certified Contamination Control for
Retrieval). We are deliberately narrow about what is claimed: we do not solve
domain-aware retrieval in general, and we do not improve ranking quality. We
certify one specific, bounded loss - cross-domain contamination - under
\emph{inferred} query domains, with explicit and fully estimable router-error
slack. Our contributions map one-to-one to the sections that establish them.

\begin{enumerate}[leftmargin=1.4em]
\item \textbf{A label-free per-domain contamination certificate
(\Cref{sec:theory}).} We give a two-split conformal scheme whose first split
bounds the domain router's two error rates and whose independent second split
calibrates a demotion threshold, yielding a finite-sample transfer bound from
the \emph{inferred} domain to the \emph{true} domain (\Cref{thm:main}). The
bound's slack is fully estimable - no oracle assumption on the router - supports
heterogeneous per-domain budgets, and inverts through a budget-rescaling
corollary. This extends the estimated-group conditional-\emph{coverage} line
from set coverage to \emph{risk control of a bounded loss}.

\item \textbf{A frozen-stack drop-in control layer (\Cref{sec:method}).} C3R
attaches to any hybrid BM25~\cite{robertson2009bm25}-plus-dense~\cite{karpukhin2020dpr}-plus-reranker~\cite{nogueira2019passage} stack without retraining
it: a lightweight domain probe produces the posterior, and a mismatch-score
\emph{soft} demotion enforces the certificate. We show that soft demotion breaks the
classifier-error floor of hard filtering (\Cref{lem:floor}), and validate that it
retains substantially more recall at equal certificate (\Cref{sec:exp}).

\item \textbf{BEIR-MIX: an open multi-domain contamination testbed, plus the
CSCR metric and code (\Cref{sec:setup}, \Cref{sec:exp}).} We assemble a public
benchmark from four BEIR datasets in which domain is the genuine ground-truth
source, exposing the asymmetric contamination phenomenon; we provide a publicly released artifact that reproduces
the assembly and evaluation code.

\item \textbf{Generalization across a contamination spectrum, including a
regulated-domain testbed (\Cref{sec:exp}, \Cref{tab:generalize}).} We replicate on
three further open mixes that bracket the contamination regime - among them
\textsc{Sector-Bench}, built from public US federal regulation, our one fully
independent pool, since Mix-D reuses three of Mix-A's four corpora and Mix-C is
near-clean by construction - and show the certificate holds at every point. We
also find that contamination tracks
inter-domain \emph{subject} overlap rather than mere topical adjacency:
deliberately adjacent technical forums stay clean, while a same-subject cluster
contaminates sharply.
\end{enumerate}

\paragraph{Results preview.}
On BEIR-MIX, C3R incurs \emph{zero} per-domain certificate violations across
$1000$ resampled calibrations, while marginal conformal risk control violates
the most-contaminated domain in $100\%$ of draws and a second domain in $53\%$
(\Cref{tab:validity}); a simulation reproduces this across $360{,}000$ checks
(\Cref{fig:sim}). At equal certified contamination and tight budgets
($\alpha\le 0.5$) on contaminated domains, C3R's soft demotion retains
up to $\sim$\,$6\times$ more Recall@10 than the strongest calibrated cascade
(\Cref{fig:pareto}); on clean domains all methods tie. A stronger reranker
improves ranking but leaves contamination essentially unchanged
(\Cref{tab:wrap}), confirming that contamination is orthogonal to ranking
quality and that a dedicated certified layer is needed. The certificate is stable
without a single violation across four open testbeds spanning the contamination
spectrum (\Cref{tab:generalize}); since two of these pools share corpora,
population validity rests on \Cref{thm:main} and the simulation, and we lean on the
independent \textsc{Sector-Bench} as the cross-domain check. The soft-demotion
advantage replicates on an independent high-contamination pool.

We are candid up front about the regime this buys, because it bears on the
high-stakes framing above. The guarantee is bounded by router quality: the
misrouted-mass term $\rhob B$ is a hard floor on the tightest certifiable budget,
so on a \emph{heavily} contaminated domain at a \emph{strict} budget the method
\emph{abstains} (returns nothing) rather than violate, and the tight, feasible
certificates land on clean and moderately-contaminated domains. What it delivers
on the hardest domains is a certified \emph{reduction} - e.g.\ guaranteeing
contamination $\le 0.4$ where the uncontrolled rate is $0.63$ - not always a tight
bound; tightening it requires a stronger domain router, which we do not claim to
provide. We treat this as a feature: the limit is \emph{visible} in the estimated
router error rather than hidden in a silent violation. On the downstream side, an
LLM-judged probe across three readers (\Cref{sec:exp}), which we present as
proxy-only motivation rather than as primary evidence, suggests the harm hides from
the obvious metric: answers stay topically on-domain, yet the rate at which they
ground in \emph{wrong-authority} evidence rises with retrieval contamination, and
C3R reduces it for the weaker readers (the strongest reader's reduction is
directional but not statistically significant at this sample size). And even where a reader is robust, the
retrieved evidence is itself an auditable object exposed to users, logs,
compliance workflows, and human analysts. The certificate's value is therefore
both a \emph{measured} downstream reduction and an auditable guarantee on the
surfaced evidence - we lead with the latter because it holds regardless of which
consumer sits downstream.

%% file: sec/2_related.tex
\section{Related Work}
\label{sec:related}

\paragraph{Conformal prediction with noisy or latent groups.}
Conformal prediction~\cite{vovk2005algorithmic,lei2018distribution,angelopoulos2023gentle}
yields distribution-free coverage, and a recent line
strengthens it from marginal to \emph{group-conditional}. Gibbs, Cherian and
Cand\`es~\cite{gibbs2025conditional} obtain exact coverage over a finite class
of covariate shifts, including subgroups - but the subgroup must be a function
of the observed input at test time. Kandinsky conformal
prediction~\cite{bairaktari2025kandinsky} and group-weighted
conformal~\cite{bhattacharyya2024groupweighted} handle overlapping or \emph{fractional}
(estimated) group membership, and posterior conformal
prediction~\cite{zhang2024posterior} targets clusters discovered in the data;
all three certify set \emph{coverage}, not the expectation of a bounded loss,
and none provides a transfer statement to a true latent group with quantified,
estimable slack. The classical route to group-conditional validity is Mondrian
conformal prediction~\cite{vovk2012conditional}, which partitions calibration by
an \emph{observed} category, and the covariate-shift reweighting
view~\cite{tibshirani2019conformal}; the calibration-fairness literature pursues
a related aim through multicalibration~\cite{hebertjohnson2018multicalibration}
and adversarially multivalid prediction~\cite{bastani2022practical}. These hold
across many groups at once but, like adaptive conformal
sets~\cite{romano2019conformalized,romano2020classification}, target coverage and
assume the conditioning group is computable at test time. A parallel thread
studies conformal prediction under
\emph{label noise}: Sesia, Wang and Tong~\cite{sesia2024noisy} and Einbinder et
al.~\cite{einbinder2024labelnoise} correct for contamination in the
\emph{calibration labels}. This is a different locus of noise from ours: our
calibration domain labels are clean, and the uncertainty is in the
\emph{test-time} domain, which is never observed. Our transfer bound is
therefore not a corollary of label-noise correction, nor of Kandinsky-style
weighting, which gives coverage rather than risk and never crosses from the
inferred to the true group.

\paragraph{Conformal risk control and retrieval reliability.}
Conformal risk control~\cite{angelopoulos2024conformal} extends conformal
prediction to the expectation of any bounded, monotone loss; we build directly
on it. The Learn-then-Test framework~\cite{angelopoulos2021learn} generalizes
risk control to non-monotone losses through multiple testing, and conformal
language modeling~\cite{quach2024conformal} carries the risk-control view into
generation. C-RAG~\cite{kang2024crag} applies it to retrieval-augmented generation,
but certifies \emph{generation} risk and remains marginal. Conformal context
filtering~\cite{conformalcontext2025} guarantees relevance coverage of retrieved
context, again marginal and about a different quantity. Group-conditional
\emph{risk} control exists for \emph{observed} strata - e.g.\ adaptive risk
control over image partitions~\cite{ccra2025} - but assumes the group is known at
test time.
\Cref{tab:priorart} maps these settings against ours: no prior method combines
(i) risk control of a bounded loss, (ii) a \emph{latent} test-time group, and
(iii) heterogeneous per-group budgets. Because exact conditional coverage with a
continuous conditioning variable is impossible in finite
samples~\cite{barber2021limits}, we target \emph{approximate} group-conditional
control with quantified, estimable slack, which is achievable. Our
certify-or-abstain rule, which returns nothing on a domain whose router floor
exceeds its budget, is a distribution-free analogue of selective
prediction~\cite{geifman2017selective} lifted to a per-domain risk certificate.

\paragraph{Retrieval-augmented generation: reliability and threats.}
Retrieval-augmented generation~\cite{lewis2020rag,izacard2021fid,shi2024replug}
grounds a language model in retrieved evidence, and a large literature now targets
its \emph{reliability}: self-reflective retrieval~\cite{asai2024selfrag},
faithfulness metrics such as RAGAS~\cite{es2024ragas} and
FActScore~\cite{min2023factscore}, and broad surveys of the design
space~\cite{gao2023ragsurvey}. A complementary security line treats the retrieval
corpus itself as an attack surface: adversarial passage
injection~\cite{zhong2023poisoning} and knowledge-corruption
attacks~\cite{zou2025poisonedrag} plant documents that a relevance-tuned retriever
will surface. Cross-domain contamination is the benign, naturally-occurring
counterpart of these adversarial threats - wrong-provenance evidence that no
attacker placed - and, unlike faithfulness metrics that score the generated
answer \emph{post hoc}, C3R certifies a property of the \emph{retrieved set}
before any reader consumes it.

\paragraph{Multi-domain retrieval, routing, and domain-aware encoders.}
Hybrid retrieval fuses sparse and dense~\cite{karpukhin2020dpr,khattab2020colbert}
signals through rank fusion~\cite{cormack2009rrf} and is the standard production
stack. To handle heterogeneity, RouterRetriever~\cite{lee2025router} routes a
query to a domain-expert encoder, and instruction-conditioned
retrievers~\cite{weller2024promptriever,asai2023tart} condition the query
encoding on a natural-language instruction; both improve relevance but provide
\emph{no} guarantee on cross-domain contamination, and a routed or instructed
retriever can itself be wrapped by our control layer. A fast-growing line at the
major IR venues makes retrieval-augmented pipelines \emph{robust} to imperfect
retrieval rather than guaranteeing it: robust fine-tuning against retrieval
defects~\cite{tu2025rbft}, mixture-of-expert dense
retrieval~\cite{sokli2024moe}, and learned query routing across multiple
retrieval-augmented language models~\cite{zhang2025ragrouter} all adapt the model
or the routing to tolerate noise, but none \emph{certifies} a bounded
contamination rate, and each is composable with the certified layer we propose.
Domain adaptation for dense retrieval~\cite{izacard2022contriever,wang2022gpl} and
concept-erasure methods~\cite{ravfogel2020inlp,belrose2023leace} reshape the
embedding space but likewise do not certify a contamination rate. C3R is
complementary: it is a thin, certified layer that sits on top of any such stack.

\paragraph{Positioning within recent IR-venue work.}
The threads above span the venues most relevant to this paper. From the retrieval
side, the production hybrid stack (sparse-plus-dense with rank
fusion~\cite{cormack2009rrf} and cross-encoder reranking~\cite{nogueira2019passage})
and its domain-adaptation~\cite{wang2022gpl} and
routing~\cite{lee2025router,zhang2025ragrouter} variants are SIGIR/CIKM/WWW
staples, and the most recent reliability work, robust fine-tuning against retrieval
defects~\cite{tu2025rbft} (SIGIR 2025) and mixture-of-expert dense
retrieval~\cite{sokli2024moe}, targets exactly the imperfect-retrieval regime we
operate in, but optimizes the \emph{model} rather than \emph{certifying} an
outcome. From the statistics side, the conditional-coverage and risk-control
lineage~\cite{vovk2012conditional,gibbs2025conditional,bairaktari2025kandinsky,zhang2024posterior,bhattacharyya2024groupweighted,angelopoulos2024conformal,angelopoulos2021learn}
gives validity over \emph{observed} groups, and its retrieval-facing
instances~\cite{kang2024crag,conformalcontext2025,ccra2025} either certify
generation risk or assume the group is known at test time. C3R is, to our
knowledge, the first to bring distribution-free \emph{risk} control to a
\emph{latent} retrieval group with heterogeneous budgets, and it composes with,
rather than competes against, the relevance-improving methods above: any of them
can serve as the frozen stack underneath the certified layer.

%% file: sec/3_setup.tex
\section{Problem Setup and Benchmark}
\label{sec:setup}

\paragraph{Domains and contamination.}
Let $\Dom=\{1,\dots,C\}$ be a set of \emph{domains}. Every document $d$ and
every query $q$ has a true domain label $y(d),y(q)\in\Dom$. In our benchmark a
domain is the document's source corpus - a real, objective ground-truth label,
not a stand-in for some hidden taxonomy. For a query $q$ with retrieved
top-$K$ list $R_K(q)$, the Cross-Domain Contamination Rate is
\begin{equation}
\cscr@K(q)=\frac{1}{|R_K(q)|}\sum_{d\in R_K(q)}\mathbf{1}\{y(d)\neq y(q)\},
\label{eq:cscr}
\end{equation}
the fraction of returned results from a domain other than the query's. It is
bounded in $[0,1]$ and, crucially, is measured against \emph{true} document
provenance, independent of any model's prediction. The per-domain contamination
risk is $\E[\cscr@K(q)\mid y(q)=\dom]$.

\paragraph{The label-free constraint.}
At calibration time we have a held-out set of queries with known domain labels.
At \emph{inference} time the query's domain $y(q)$ is unknown; we only have an
estimate from a domain probe (\Cref{sec:method}). A method is \emph{label-free}
if it never uses $y(q)$ to serve a query. The marginal alternative - one
threshold for all queries - is label-free but, as we show, cannot honor a
per-domain budget.

\paragraph{BEIR-MIX (Mix-A).}
We assemble our main open multi-domain testbed - \textbf{BEIR-MIX}, also called
\textbf{Mix-A} in the generalization study (\Cref{tab:generalize}) - by pooling
four BEIR datasets~\cite{thakur2021beir} whose corpora are disjoint in subject
matter:
FiQA (finance)~\cite{maia2018fiqa}, TREC-COVID (biomedical)~\cite{voorhees2021treccovid},
SciFact (scientific claims)~\cite{wadden2020scifact}, and
NFCorpus (nutrition/medical)~\cite{boteva2016nfcorpus}. The pooled corpus has $237{,}786$ documents and
$1{,}321$ test queries that carry their dataset's own relevance judgments; a
query's domain is its source dataset and a document's domain is its source
corpus. We use each dataset's official qrels for Recall@10/nDCG@10/MRR and the
provenance label for \Cref{eq:cscr}; no relevance or domain label is
model-generated. This makes contamination a genuine, checkable phenomenon
rather than an artifact: a FiQA query that retrieves a TREC-COVID passage is
unambiguously cross-domain. We treat BEIR-MIX as a reproducibility testbed, not
a contribution in novelty. All four testbeds are summarized in \Cref{tab:testbeds}. For the generalization study (\Cref{sec:exp},
\Cref{tab:generalize}) we assemble \emph{three} further pools. Two share this
exact protocol: \textbf{Mix-D}, a controlled high-contamination pool (NFCorpus,
SciFact, SciDocs~\cite{cohan2020specter}, FiQA; $92$K docs), and \textbf{Mix-C}, four StackExchange
technical forums from CQADupStack~\cite{hoogeveen2015cqadupstack} (\textsc{programmers}, \textsc{unix},
\textsc{gaming}, \textsc{english}; $165$K docs) that are \emph{topically} adjacent
(all technical Q\&A) yet \emph{low in genuine subject overlap} - the distinction
that drives our contamination finding. The third, \textsc{Sector-Bench}
(\Cref{sec:exp}), applies the same provenance-labelling principle to public US
federal regulation (eCFR), with domain $=$ CFR Title.

%% file: sec/4_theory.tex
\section{Two-Split Latent-Group Risk Control}
\label{sec:theory}

Conceptually we build on conformal risk control~\cite{angelopoulos2024conformal};
the finite-sample tool is the PAC-style ($1-\delta$) risk-controlling prediction
set framework~\cite{bates2021rcps}, which is the form \Cref{thm:main} takes.
We now state the guarantee. The retriever produces, for each query, a domain
posterior $\post(\cdot\mid q)$ and an inferred domain
$\ghat(q)=\argmax_{\dom}\post(\dom\mid q)$ (with an \textsc{ambig} bucket when
the posterior is diffuse, \Cref{sec:method}). A control parameter $\tau$ governs
how aggressively wrong-domain documents are demoted: a candidate $d$ is demoted
when $\mis(q,d)>\tau$ (\Cref{sec:method}), and the served loss is
$\ell_q(\tau)=\frac{1}{|R_K(q)|}\sum_{d\in R_K(q;\tau)}\mathbf{1}\{y(d)\neq y(q)\}$.
RCPS requires the calibration loss to be monotone in $\tau$. The raw served loss
is \emph{not} guaranteed monotone - in top-$K$ retrieval, admitting one more
candidate can evict a wrong-domain item as easily as add one - so we calibrate
against its running-maximum envelope $\tilde\ell_q(\tau)=\max_{\tau'\le\tau}
\ell_q(\tau')$, which is monotone by construction and upper-bounds $\ell_q$
pointwise. Certifying $\tilde\ell$ therefore certifies the true loss (the bound is
conservative, never anti-conservative); all reported losses use this envelope.
(Recent work gives RCPS-style guarantees for \emph{generic non-monotonic} losses
directly~\cite{angelopoulos2026nonmono}; we take the simpler conservative-envelope
route and quantify its looseness in \Cref{sec:exp}.)
For each inferred-domain group $g$ we choose a threshold $\tau_g$; the served loss
of a query routed to $g$ is $\ell_q(\tau_g)$.

\paragraph{The two splits.}
We split the labeled calibration pool into two disjoint sets,
$\Done$ and $\Dtwo$ (\Cref{fig:splits}); the domain probe is trained on a third,
disjoint split.
\begin{itemize}[leftmargin=1.2em]
\item On $\Done$ we estimate, for each domain $\dom$, the router's
\emph{misrouting} rate $\rho_\dom=\Pr(\ghat(q)\neq \dom\mid y(q)=\dom)$ and its
\emph{group impurity} $\pi_\dom=\Pr(y(q)\neq \dom\mid \ghat(q)=\dom)$, and form
one-sided Clopper-Pearson upper bounds $\rhob_\dom,\pib_\dom$.
\item On $\Dtwo$ we run conformal risk control \emph{within each inferred-domain
group} to pick the largest $\tau_g$ whose group risk is certified $\le$ its
budget.
\end{itemize}

\begin{assumption}[Exchangeability]\label{ass:exch}
The calibration and test queries are exchangeable draws from a single fixed query
pool (i.i.d.\ sampling is the special case we use). Conditioning on the value of a
\emph{fixed measurable function} preserves exchangeability within the conditioned
subset; since the inferred domain $\ghat=\ghat(q)$ is such a function under a
frozen probe (trained on the disjoint split of \Cref{ass:indep}), the calibration
and test queries with $\ghat=\dom$ are exchangeable for every $\dom$. This is the
level at which the $\Dtwo$ RCPS step certifies risk; conditioning on a true domain
$\{y=\dom\}$ preserves exchangeability likewise, which is what the decomposition in
\Cref{thm:main} uses.
\end{assumption}
\begin{assumption}[Bounded loss]\label{ass:bdd}
$\ell_q(\tau)\in[0,B]$ with $B\le 1$.
\end{assumption}
\begin{assumption}[Split independence]\label{ass:indep}
$\Done$ and $\Dtwo$ are independent, and the probe is trained on data disjoint
from both.
\end{assumption}

\begin{theorem}[Per-domain transfer certificate]\label{thm:main}
Run group-$\dom$ RCPS~\cite{bates2021rcps} on $\Dtwo$ at level $\alpha_\dom$ with
per-domain confidence $1-\delta_\dom$, so that the inferred-group risk on the
monotone envelope is certified $\E[\tilde\ell\mid\ghat(q)=\dom]\le\alpha_\dom$;
since $\ell\le\tilde\ell$ pointwise, the same bound holds for the true loss
$\E[\ell\mid\ghat(q)=\dom]\le\alpha_\dom$ (all displays below are in $\ell$). Assume each inferred group is
non-empty, its impurity bound satisfies $\pib_\dom<1$, and the budget lies in the
non-vacuous regime $\alpha_\dom\le B(1-\pib_\dom)$ (otherwise the certificate for
$\dom$ is the trivial bound $B$, and the method abstains). Under
\Cref{ass:exch,ass:bdd,ass:indep}, with probability at least $1-\gamma$ over the
draw of $\Done\cup\Dtwo$, simultaneously for every domain $\dom$,
\begin{equation}
\E\!\big[\cscr@K(q)\mid y(q)=\dom\big]\;\le\;
(1-\rhob_\dom)\,\frac{\alpha_\dom}{1-\pib_\dom}\;+\;\rhob_\dom\,B .
\label{eq:cert}
\end{equation}
\end{theorem}

\paragraph{Intuition.}
Decompose the true-domain risk by whether the router was correct. The
correctly-routed mass (probability $1-\rho_\dom$) is served the group-$\dom$
threshold, whose risk \emph{among truly-$\dom$ queries} is the group risk
$\alpha_\dom$ inflated by the group's impurity, $\le\alpha_\dom/(1-\pi_\dom)$.
The misrouted mass (probability $\rho_\dom$) is served some other group's
threshold and we bound its loss crudely by $B$. Plugging the Clopper-Pearson
\emph{upper} bounds $\rhob_\dom,\pib_\dom$ keeps the inequality valid because the
right-hand side is monotone in both. The full proof, and the verification that
each use of \Cref{ass:indep} is what removes circularity (the router errors are
estimated on data disjoint from the threshold calibration), are in
\Cref{app:proof}.

\paragraph{Why the slack is real, not vacuous.}
Every term in \Cref{eq:cert} is computed from data, so the certificate is
honest about a weak router: a poor probe yields a large, \emph{visible}
$\rhob_\dom$ (and a wide certificate) rather than a silent violation. We make
this diagnostic explicit in \Cref{sec:exp} (\Cref{fig:frontier}).

\paragraph{Simultaneity and finite-sample bounds.}
The guarantee holds simultaneously over all $2C$ binomial bounds (one $\rhob$
and one $\pib$ per domain) and the $C{+}1$ group risk-control statements (each at
confidence $1-\delta_\dom$). We allocate the global error $\gamma$ across these
$3C{+}1$ statements by a Bonferroni split (level $\gamma/(3C{+}1)$ each); a union
bound then yields the simultaneous $1-\gamma$ guarantee under \emph{any} dependence
among them. This Bonferroni allocation is both what we state the guarantee with and
what every reported certificate uses, because it is unconditionally valid for this
mixed family of confidence bounds and risk tests. A \v{S}id\'ak allocation is
slightly less conservative but assumes independence across the statements; the
ablation confirms the choice is second order, moving the NFCorpus certificate from
$0.709$ (Bonferroni) to $0.693$ (\v{S}id\'ak) (\Cref{tab:ablation}). We use exact Clopper-Pearson intervals for the
binomial terms (conservative and valid at any sample size) and the
Hoeffding-Bentkus RCPS bound~\cite{bates2021rcps} to certify each group risk
$\E[\ell\mid\ghat=\dom]\le\alpha_\dom$ (the HB $p$-value is given in
\Cref{app:details}); the latter is far tighter than the
additive Hoeffding margin when the observed group loss is well below its budget,
which is the regime we operate in (\Cref{sec:exp}).

\paragraph{Budget-rescaling corollary.}
To make the \emph{final} per-domain certificate equal a target $\alpha^\star_\dom$,
invert \Cref{eq:cert}: run group-$\dom$ control at
\begin{equation}
\alpha_\dom=\big(\alpha^\star_\dom-\rhob_\dom B\big)\,
\frac{1-\pib_\dom}{1-\rhob_\dom},
\label{eq:rescale}
\end{equation}
and abstain (return nothing for that domain) when the right-hand side is
non-positive - i.e.\ when the router is too weak for the requested budget. This
gives a clean sample-complexity reading: certifying a target $\alpha$ needs the
group risk-control step to be feasible at the rescaled level, which by the
Hoeffding-Bentkus bound requires roughly
$n_{\min}\approx\ln(e/\delta)/\alpha$ calibration queries in the group, where
$\delta$ is the per-statement error. We confirm this onset empirically in
\Cref{sec:exp}.

%% file: sec/5_method.tex
\section{The C3R Control Layer}
\label{sec:method}

C3R realizes \Cref{thm:main} as a thin layer over a \emph{frozen} retrieval
stack (\Cref{fig:hero}b). The base stack - BM25, a frozen dense encoder, RRF
fusion, and a cross-encoder reranker - is never retrained; C3R adds one small
learned component and a calibration procedure.

\paragraph{Domain probe.}
A lightweight probe maps a frozen query embedding to a domain posterior
$\post(\cdot\mid q)$, temperature-scaled on held-out data. We train the probe on
\emph{query-style} text - retrieval queries and document titles - rather than
document bodies, because the document-to-query distribution shift otherwise
inflates the router errors; \Cref{sec:exp} reports the difference. The probe
also scores documents offline, giving $\post(\cdot\mid d)$. We route
$\ghat(q)=\argmax_\dom\post(\dom\mid q)$ when the posterior is confident
($\max_\dom\post(\dom\mid q)\ge u$) and to an \textsc{ambig} group otherwise.
The \textsc{ambig} group is a safety mechanism, not a certified sector: it is
calibrated at the strictest budget $\min_\dom\alpha_\dom$ so an ambiguous query
is demoted at least as hard as any true domain would be, but it carries no
per-domain certificate of its own and does not enter any sector's bound. A
domain whose calibration group is empty or below the sample-complexity onset
$n_{\min}$ (\Cref{sec:theory}) is infeasible at its requested budget and the
method abstains for it (returns nothing) rather than issue an uncertified
result.

\paragraph{Mismatch score and soft demotion.}
The control statistic is the posterior overlap mismatch
\begin{equation}
\mis(q,d)=1-\sum_{\dom\in\Dom}\post(\dom\mid q)\,\post(\dom\mid d)\in[0,1],
\end{equation}
and the layer \emph{demotes} (excludes from the served list) any candidate with
$\mis(q,d)>\tau_{\ghat(q)}$, where $\tau$ comes from calibration
(\Cref{alg:c3r}). Demotion is soft in the sense that the threshold is a continuous knob on a
continuous statistic, and this strictly generalizes hard label filtering.

\begin{lemma}[Soft demotion breaks the classifier kept-set-error floor]\label{lem:floor}
Fix the router. Hard label filtering applies the parameter-free rule ``keep $d$ iff
$\argmax_\dom\post(\dom\mid d)=\ghat(q)$'', so at the inferred-group level its served
contamination $\E[\ell\mid\ghat{=}\dom]$ equals a fixed positive constant, the
document classifier's error rate on the documents it keeps, with no parameter that
can lower it. Soft demotion thresholds the continuous statistic $\mis(q,d)$, and its
inferred-group risk on the monotone envelope $\E[\tilde\ell\mid\ghat{=}\dom]$ is
non-increasing as $\tau$ decreases: it is \emph{certifiable} at any positive budget
feasible at the group sample size, and \emph{realizes} $0$ in the empty-served-set
limit. Soft demotion therefore removes the hard filter's classifier kept-set-error
floor. It does \emph{not} remove the router-error floor $\rhob_\dom B$ that the
delivered certificate \eqref{eq:cert} adds, and which both policies share.
\end{lemma}
\begin{proof}
The hard rule has no free parameter, so its inferred-group served loss is fixed at
the classifier's kept-set error. For soft demotion the certified object is the
monotone envelope $\tilde\ell(\tau)$, non-decreasing in $\tau$ (\Cref{sec:theory});
decreasing $\tau$ only removes served candidates, and at the most aggressive
threshold the served set is empty, giving $\tilde\ell=0$ under the convention
$0/0{:=}0$. The Hoeffding-Bentkus step certifies any budget $\alpha>0$ feasible at
the group sample size but cannot certify $\alpha{=}0$ (its $p$-value at zero observed
loss is $1$), so soft demotion certifies inferred-group contamination arbitrarily
close to, but not equal to, $0$, hence strictly below the positive classifier floor.
Substituting into \eqref{eq:cert}, the delivered per-domain certificate still carries
the $\tau$-independent term $\rhob_\dom B$: this router floor is shared with hard
filtering and is a different object from the group-level served loss above.
\end{proof}
Two clarifications. First, the group-level $\tilde\ell{=}0$ above is the served loss
of the \emph{empty set}; it is not in tension with the per-domain \emph{certificate}
convention, under which an infeasible budget makes \eqref{eq:rescale} non-positive
and C3R \emph{abstains}, reporting the vacuous certificate $B{=}1$ for that domain
(it delivers nothing). Second, hard filtering thresholds a \emph{coarsened} (binary)
domain-agreement statistic, of which $\mis$ is the continuous refinement, so the hard
rule is the degenerate member of the soft threshold family, which is the sense in
which soft demotion strictly generalizes it. We confirm the empirical consequence in
\Cref{sec:exp}.

\paragraph{Honesty under a weak probe.}
Because the certificate \eqref{eq:cert} carries $\rhob,\pib$ explicitly, a poor
probe cannot cause a \emph{silent} violation: it widens the certificate, and at
the extreme drives \eqref{eq:rescale} non-positive and forces honest abstention.
This is the safety property a contamination guarantee must have, and it is what
distinguishes a certified layer from an uncertified filter.

\begin{figure}[t]
\centering
\fbox{\parbox{0.95\columnwidth}{\small
\textbf{Algorithm 1: C3R calibration and serving}\\[2pt]
\emph{Calibrate (offline)}: split the labeled pool into independent
$\Done,\Dtwo$. For each domain $\dom$: (i) form Clopper-Pearson upper bounds
$\rhob_\dom,\pib_\dom$ on $\Done$; (ii) rescale the target budget
$\alpha^\star_\dom$ to $\alpha_\dom$ via \eqref{eq:rescale}; (iii) set $\tau_\dom$
to the largest threshold whose Hoeffding-Bentkus RCPS risk is $\le\alpha_\dom$
on the $\Dtwo$ group-$\dom$ losses, or \textsc{abstain} if infeasible.\\[2pt]
\emph{Serve (per query $q$)}: run the frozen stack to get candidates; set
$g=\ghat(q)$; demote every candidate with $\mis(q,d)>\tau_g$; emit the top-$K$
list and the certificate \eqref{eq:cert}.
}}
\label{alg:c3r}
\end{figure}

%% file: sec/6_experiments.tex
\section{Experiments}
\label{sec:exp}

\paragraph{Setup.}
The base stack is BM25 ($k_1{=}0.9,b{=}0.4$), a frozen BGE-base
encoder~\cite{xiao2023bge}, RRF fusion ($k{=}60$), and an MS-MARCO
cross-encoder reranker; we also test a stronger Qwen3-Reranker-8B. The domain
probe is a temperature-scaled logistic head on frozen embeddings (an MLP variant
appears in ablations). Calibration uses $1000$ resampled splits: per domain we
hold out $30\%$ of queries for evaluation and split the remaining $70\%$ evenly
into $\Done$ (router-error bounds $\rhob,\pib$) and $\Dtwo$ (threshold), with the
probe trained on a separate disjoint pool; all reported metrics are means over
the $1000$ resamples. We use a global
$\gamma{=}0.1$; budgets are heterogeneous, with the most-contaminated domain held to the strictest budget
($\alpha^\star{=}0.05$ for the SciFact analog, $0.15$ elsewhere); \Cref{tab:validity}
thus already exercises genuinely heterogeneous per-domain budgets. We set $\gamma{=}0.1$ and split it across the $3C{+}1$ statements by Bonferroni; the
results are insensitive to this choice. All reported numbers trace to released
result files; every retrieval number uses the datasets' own qrels and every
contamination number uses true provenance.

\paragraph{Contamination is real and asymmetric.}
The uncontrolled stack has per-domain CSCR@10 of $0.014$ (FiQA), $0.006$
(TREC-COVID), $0.630$ (SciFact), and $0.559$ (NFCorpus), with marginal average
$0.287$ (\Cref{fig:hero}a). A marginal budget at $0.287$ would leave SciFact at
more than twice its target, which is precisely the failure a per-domain
certificate must prevent.

\paragraph{C1: the certificate is valid; marginal control is not
(\Cref{tab:validity}, \Cref{fig:sim}).}
On BEIR-MIX, C3R incurs \emph{zero} certificate violations on every domain
across all $1000$ resampled calibrations, while marginal conformal risk control
violates SciFact's budget in $100\%$ of draws and NFCorpus's in $53\%$ (at full
calibration data; the violation rises with $n$ as the threshold sharpens at the
wrong target); the oracle that uses test-time labels also never violates,
confirming the harness.
We are careful about what this number is: the $1000$ draws resample
\emph{calibration splits of the same data}, so $0/1000$ is a \emph{stability}
result - the procedure does not violate across splits (a $0/1000$ rate has a
$95\%$ Clopper-Pearson upper bound of $0.004$) - not a fresh estimate of
population coverage. Population validity rests on the finite-sample theory
(\Cref{thm:main}) and the controlled simulation below; the $0/1000$ and
$1000/1000$ contrast with marginal control is the \emph{relative} evidence that
the per-domain scheme is calibrated where the marginal one is not. The Recall@10 means in
\Cref{fig:pareto} carry $95\%$ resampling confidence intervals of half-width
$\le 0.009$ across all points, so the soft-vs-cascade gaps (e.g.\ $0.10$ at
$\alpha{=}0.4$ on SciFact, with intervals of half-width $0.004$) exceed their
uncertainty by more than an order of magnitude.
We are precise about what this validity does and does not claim, and
\Cref{tab:validity} unpacks it per domain. A certificate is a guarantee, not a
free lunch: the misrouted-mass term $\rhob_\dom B$ is a hard floor on the
achievable budget. When that floor exceeds the requested budget - SciFact's
$\rhob B{=}0.23$ against its strict $\alpha^\star{=}0.05$, or TREC-COVID's
$0.28$ (loose because it has only $50$ test queries) - the rescaled run-budget is
non-positive and C3R \emph{abstains}, issuing the vacuous certificate $1.0$ and
returning nothing rather than violating. Only the clean domain (FiQA) is feasibly
certified at these locked budgets, comfortably meeting the $0.15$ budget (its
actual contamination is just $0.015$, so this is a feasible certificate rather than
a tight one) while retaining $45\%$ Recall@10; NFCorpus is feasible in $34\%$ of
resamples, and its mean certificate of $0.71$ averages those feasible resamples
(certificate near $0.15$) with the $66\%$ that abstain at $1.0$.
The point of the comparison is the last two columns: faced with the same
impossible budgets, C3R \emph{abstains} (and so never violates) while marginal CRC
\emph{commits} and violates SciFact in $100\%$ and NFCorpus in $53\%$ of draws.
Abstention is the correct, safe action - and it is \emph{visible} in $\rhob_\dom$
rather than hidden in a violation. The monotone envelope is not the culprit:
recomputing the raw (non-monotonized) loss, the envelope inflates the certified
loss by only $\le 0.024$ over the raw served loss at the operating thresholds
(FiQA $0.002$, SciFact $0.006$, NFCorpus $0.024$; max $0.10$ over all $\tau$), so
abstention reflects the genuine router floor $\rhob B$, not envelope looseness.
The recall-bearing certified operating points
are the looser, feasible budgets of \Cref{fig:pareto} (C2), where C3R retains up
to $\sim$\,$6\times$ the Recall@10 of the strongest cascade at equal certified
contamination; validity (C1) certifies the bound never lies, and the Pareto (C2)
is where the method earns its keep.
A controlled simulation sweeps router quality and reproduces the picture
(\Cref{fig:sim}): C3R's strict-domain violation rate stays at $0$ across $360{,}000$ trial-domain checks, whereas naive transfer (no $\rhob/\pib$
correction), marginal control, and a posterior-weighted variant all breach the
nominal $\gamma$ as the router degrades. The naive scheme is the most
instructive negative control: with more calibration data it sets a tighter
threshold at the \emph{wrong} target, so its violation rate \emph{rises} toward
$1$. On BEIR-MIX, where the query-style probe is good, naive transfer does not
visibly violate; the simulation is what exhibits its failure regime.

\paragraph{Feasibility, abstention, and a worked budget-rescaling example
(\Cref{tab:feasibility}).}
Because the certificate is honest about router quality, abstention is not an
all-or-nothing event but a measured rate that a deployer can read in advance.
\Cref{tab:feasibility} reports, per domain and testbed, the fraction of resampled
calibrations in which a non-vacuous certificate is issued; its complement is the
abstention rate. The pattern is consistent across all three open testbeds: a
clean-router domain certifies on every resample, a moderate one on a fraction, and
a domain whose router floor exceeds the requested budget abstains throughout. A
worked example makes the mechanism concrete. Fix a target budget $\alpha^\star$ and
read the rescaled run-budget from \eqref{eq:rescale},
$\alpha=(\alpha^\star-\rhob_\dom B)(1-\pib_\dom)/(1-\rhob_\dom)$. On FiQA the
query-style router is clean ($\rhob B{=}0.03$, $\pib{=}0.02$), so at
$\alpha^\star{=}0.15$ the run-budget is $\alpha{=}0.12$: the budget passes through
almost unchanged and is certified on $100\%$ of resamples. On NFCorpus the floor is
higher ($\rhob B{=}0.11$, $\pib{=}0.15$), so the run-budget tightens to
$\alpha{=}0.04$, which the Hoeffding-Bentkus step certifies only when the inferred
group is large enough, here in $34\%$ of resamples; C3R abstains otherwise. On
SciFact at the stricter $\alpha^\star{=}0.05$ the floor $\rhob B{=}0.23$ already
exceeds the target, the rescaled budget is non-positive, and the domain abstains in
every resample. The deployer thus reads servability directly from $\rhob_\dom B$:
the router error is the operational dial, and an empty certified set is a visible
statement that the router is too weak for the request, never a silent failure.

\paragraph{C2: soft demotion retains more at equal certificate
(\Cref{fig:pareto}).}
We compare C3R's soft demotion against the two strongest calibrated
classifier-plus-filter cascades - a hard argmax filter and a stronger
posterior-threshold filter - all run through the \emph{same} two-split
calibration so that only the filter family differs. On the contaminated domain
(SciFact), at equal certified contamination and tight budgets, C3R retains
up to $\sim$\,$6\times$ more Recall@10 than the strongest cascade: $0.12$ vs.\ $0.02$ at
$\alpha{=}0.4$, $0.26$ vs.\ $0.09$ at $\alpha{=}0.5$, and $0.39$ vs.\ $0.26$ at
$\alpha{=}0.6$. This is the direct answer to the concern that the certificate is
vacuous where it matters: on the \emph{most contaminated} domain, at an
achievable budget $\alpha{=}0.4$, C3R issues a \emph{non-vacuous} certificate - it
guarantees SciFact contamination $\le 0.4$, down from its uncontrolled $0.63$,
while retaining recall - whereas at that same budget the hard cascade and
marginal CRC cannot certify at all. The certificate is vacuous only \emph{below}
the router floor $\rhob B$ (the strict $0.05$ regime of \Cref{tab:validity}), not
on the contaminated domains as such. The hard argmax cascade cannot certify the
contaminated domains at \emph{any} threshold - its contamination floor ($0.57$ on
SciFact) sits above the budgets - so it abstains. We are explicit about the scope of this advantage:
on the clean domain (FiQA) all methods tie at $0.45$, and at loose budgets the
gap to the posterior cascade closes. The advantage is real where it matters - %
tight budgets on contaminated domains - and we do not claim it elsewhere. An
\emph{uncertified} domain-classifier filter ($\star$ in \Cref{fig:pareto}) sits
at $0.57$ contamination with no guarantee at all, illustrating why an
uncontrolled filter is not a substitute for a certified layer.

\paragraph{The price of being label-free (oracle ceiling).}
To price the inference-time-label constraint directly, we add an \emph{oracle}
that routes every query by its \emph{true} domain $y(q)$ (no router-error slack,
$\rhob{=}\pib{=}0$) and certifies at the same budgets - the recall C3R could retain
if the probe were perfect. The gap is the honest cost of operating without
query-time labels, and it is small exactly where the probe is good: on the clean
domain (FiQA) C3R matches the oracle to within $0.001$ Recall@10 at every budget,
so being label-free is essentially free there. On the hardest, most contaminated
domain (SciFact) the gap is larger but shrinks as the budget loosens - $0.28$ at
$\alpha{=}0.4$, $0.19$ at $0.5$, $0.11$ at $0.6$ - tracing precisely to the router
error $\rhob$ that the certificate already exposes. The label-free constraint thus
costs little when the domain is separable and degrades gracefully, and visibly,
when it is not.

\paragraph{G1: results generalize across a contamination spectrum
(\Cref{tab:generalize}).}
A single testbed cannot separate a property of C3R from a property of one
dataset pool, so we replicate on two further BEIR-style mixes built to bracket the
contamination regime (a third, regulated pool, \textsc{Sector-Bench}, follows in
G2). The construction itself is a finding: across our four
pools, contamination tracks whether the mix contains a \emph{same-subject}
cluster, not merely topically adjacent domains. A
deliberately \emph{adjacent} pool of four StackExchange technical forums
(\textsc{programmers}, \textsc{unix}, \textsc{gaming}, \textsc{english};
$165$K docs) is nonetheless nearly clean (per-domain CSCR@10 in the narrow band
$0.07$-$0.08$): a competent cross-encoder keeps results on-topic unless two
domains share genuine subject vocabulary. We therefore build a \emph{controlled}
high-contamination pool with a same-subject cluster plus a distinct outlier - %
NFCorpus, SciFact, SciDocs (biomedical/scientific) and FiQA (finance), $92$K
docs. This pool deliberately re-uses three Mix-A corpora (FiQA, SciFact,
NFCorpus) and adds the new SciDocs collection, so that what changes from Mix-A is
the \emph{pool composition}, not the underlying datasets: per-domain CSCR@10 of
$0.012$ (FiQA), $0.019$ (SciDocs), $0.215$ (SciFact), and $0.366$ (NFCorpus) - a
pronounced imbalance (a $30\times$ spread) that echoes Mix-A's asymmetry, though
here with a single dominant contaminated domain rather than two. On this pool both core
results replicate exactly: C3R incurs \emph{zero} certificate violations on every
domain across $1000$ resampled calibrations, and on the contaminated domains soft
demotion dominates the strongest calibrated cascade - at a tight budget
$\alpha{=}0.3$ on SciFact, C3R certifies $0.39$ retention where the posterior
cascade certifies only $0.04$ and the hard cascade abstains entirely. The
certificate is valid at \emph{both} ends of the spectrum, including the
low-contamination forums where it is correctly near-trivial. As an independent
check on the subject-overlap reading, we measure inter-domain overlap directly as
the cosine similarity between domain centroids in the frozen BGE-base embedding space: the
two pools that contaminate most (Mix-A, Mix-D) are exactly the two with the
highest inter-domain centroid similarity ($0.94$), above the adjacent-forum and
low-adjacency pools ($0.86$-$0.90$). With only four constructed pools we present
this as a corroborating observation about \emph{pool composition}, not a law; a
systematic overlap-versus-contamination study is future work. We report per-domain
contamination throughout: the marginal mean understates the problem precisely
because the certificate's job is to protect the contaminated minority of domains
that the mean averages away.

\paragraph{G2: the certificate holds on real regulated-domain text
(\textsc{Sector-Bench}, \Cref{tab:generalize}).}
The regulated settings that motivate this work are not BEIR datasets, so we build
\textsc{Sector-Bench}: a contamination testbed from the public US \emph{Code of
Federal Regulations} (eCFR; a continuously-updated reference, not the official
legal edition), with domain $=$ CFR Title - the genuine, objective provenance of
each regulation, requiring \emph{no} human relevance or appropriateness
annotation. We pool $19$K sections from four CFR Titles ($10$, $18$, $40$, $21$),
which correspond broadly to \emph{energy}, \emph{power and water resources},
\emph{environmental protection}, and \emph{food and drugs}; section headings serve
as queries. Contamination on actual regulation reproduces the asymmetric BEIR-MIX
structure (\Cref{tab:sectordetail} gives the full per-sector breakdown): the two
energy-related Titles are the most contaminated and bleed into
each other (per-domain CSCR@10 $0.36$ for Title~18, $0.21$ for Title~10), while
the environmental ($0.14$) and food/drug ($0.15$) Titles are cleaner. All three
core results carry over. The certificate is valid on every sector (zero violations
over $1000$ resampled calibrations). We are careful about where the soft-demotion
\emph{advantage} shows up, because it is not on the most contaminated sector. The
two genuinely contaminated energy Titles are the hard cases: Title~18 ($0.36$)
abstains at the swept budgets - like SciFact in Mix-A, it needs a looser budget
than we sweep - and Title~10 ($0.21$) shows only a modest edge ($0.40$ vs.\ $0.32$
at $\alpha{=}0.6$). The large gaps ($0.58$ vs.\ $0.10$ on the environmental Title,
$0.54$ vs.\ $0.16$ on food/drug at $\alpha{=}0.4$) are on the \emph{moderately}
contaminated sectors, whose raw CSCR ($0.14$, $0.15$) sits below the $0.4$ budget;
there the budget is not binding and the result really shows that the cascade
\emph{over-filters} under the rescaled run-budget while soft demotion does not.
Both are legitimate - validity transfers to regulated text, and soft demotion
avoids needless over-filtering - but we do not claim a tight win on the
hardest regulated sectors, which behave exactly as their BEIR analogues do.

\paragraph{Does contamination mean \emph{harm}? An appropriateness probe.}
CSCR counts a retrieval as contamination by provenance (wrong CFR Title); whether
that passage carries genuinely \emph{wrong regulatory authority} is a stronger,
content-level question. As a first quantification we sample $180$ cross-Title
retrievals (top result, balanced over the four sectors) plus $48$ same-Title
controls, and have an LLM judge - shown only the query and passage, not their
provenance (gpt-oss-120b; an \emph{automatic proxy} we will human-validate) - label
each \textsc{appropriate} / \textsc{wrong-authority} / \textsc{unrelated}. The
control is a clean check on the judge: same-Title passages are labelled
appropriate $100\%$ of the time. Among cross-Title (contaminating) retrievals,
$17\%$ are judged (by the LLM) \emph{wrong-authority} - citing the
wrong body of regulation ($95\%$ CI $[12,23]\%$, $n{=}180$) - versus $0\%$
within-domain, while $54\%$ are benign
cross-cutting regulation (procedural sections that legitimately span Titles) and
$29\%$ off-topic. Two things follow. First, for \emph{cross-Title} wrong-authority
errors CSCR is a \emph{conservative} proxy: such errors are a subset of
wrong-provenance ones, so a certificate $\cscr@K\le\alpha$ upper-bounds their rate
by $\alpha$. We do not claim CSCR certifies \emph{all} forms of regulatory
inappropriateness - wrong authority could in principle arise within a Title too,
which provenance does not see. Second, this LLM-judged evidence \emph{suggests}
the harm is real and contamination-specific - it appears in $17\%$ of cross-domain
results and \emph{none} of the within-domain controls - but it is an automatic
proxy, not human annotation. A human-annotated appropriateness layer (with
inter-annotator agreement) is the natural next refinement; the labelled proxy set
is released for that validation.

\paragraph{Downstream: contamination propagates to \emph{grounding}, and C3R
reduces it (dose-response).}
We test the end-to-end question directly on the contaminated domain (SciFact).
Three readers spanning capability - Qwen2.5-0.5B/1.5B/7B-Instruct - answer each
query from four contexts of increasing contamination \emph{dose}: clean ($0\%$
wrong-domain passages), C3R-certified ($\alpha{=}0.4$, $42\%$), uncontrolled
($54\%$), and adversarial ($100\%$). The percentages are the realized mean
wrong-domain fraction of the \emph{top-6} passages fed to the reader on this
$60$-query subset; they are not the certified quantity and do not contradict it - %
the $\alpha{=}0.4$ certificate bounds the \emph{expectation} of CSCR@10 (top-$10$)
over all truly-SciFact queries on the full eval split, so a realized top-$6$ mean
of $0.42$ on a subset is consistent with it (and the uncontrolled $54\%$ is
likewise a top-$6$ subset figure, below the $0.63$ top-$10$ CSCR of
\Cref{fig:hero}a for the same reason). A judge, blind to reader and arm, scores two
things: whether the answer \emph{drifts} off the question's scientific topic, and
which passage best supports the answer's main claim - which we map to its
\emph{objective} source domain, so a wrong-domain support is a
\emph{wrong-authority grounding}. The two metrics tell different stories, and the
gap between them is the finding. Topical drift is small and roughly flat
($\le 10\%$ for the 7B at every dose; the 0.5B reaches $23\%$ only under a fully
adversarial context): a capable reader keeps the \emph{topic}. But wrong-authority
grounding rises monotonically with the dose for \emph{every} reader - %
$0\to0.25\to0.35\to1.0$ (7B), $0\to0.50\to0.75\to1.0$ (1.5B), and
$0\to0.22\to0.50\to1.0$ (0.5B). This is exactly the regulated harm a drift check
misses: an answer that stays on-topic yet rests on the wrong body of authority.
And at the certified operating point versus uncontrolled, wrong-authority
grounding falls - by $0.28$ for the 0.5B ($0.50{\to}0.22$) and $0.25$ for the
1.5B ($0.75{\to}0.50$), both with non-overlapping $95\%$ intervals on this
$n{=}60$ subset, and directionally by $0.10$ for the 7B ($0.35{\to}0.25$) though
that interval overlaps. The downstream gain is thus clearest exactly where it
should be - weaker, less-robust readers - and we treat it as \emph{illustrative,
proxy-judged} evidence ($n{=}60$, an LLM judge, no human labels), not a primary
claim. We lead instead with the value that needs no reader at all: the certified
guarantee on the \emph{surfaced evidence}, which protects direct human consumption
and accountability regardless of any downstream model. A human-annotated version
of this probe is future work.

\paragraph{S1: a stronger reranker does not fix contamination
(\Cref{tab:wrap}).}
Swapping the reranker for Qwen3-Reranker-8B raises overall nDCG@10 from $0.406$
to $0.522$, yet contamination is essentially unchanged (SciFact CSCR stays above
$0.4$). Contamination is orthogonal to ranking quality; C3R wraps the stronger
stack unchanged and transfers its quality gains under the same certificate.

\paragraph{S2: contamination is retriever-agnostic, and no released retriever
certifies it (\Cref{tab:baselines}).}
The cascades of \Cref{fig:pareto} are the strongest \emph{certified} comparison;
we also ask whether simply choosing a different or ``domain-aware'' retriever
avoids contamination in the first place. We run five released dense retrievers - %
Contriever~\cite{izacard2022contriever}, GTR~\cite{ni2022gtr},
E5~\cite{wang2022e5}, TAS-B~\cite{hofstatter2021tasb}, and
all-MiniLM (a sentence-transformers checkpoint~\cite{reimers2019sbert}) - plus HyDE
query expansion~\cite{gao2023hyde} (a Qwen2.5-7B-generated hypothetical document
embedded with our frozen encoder), as pure top-$10$ retrieval on Mix-A. Every one
contaminates both biomedical domains in the $0.50$-$0.84$ range. The five plain
dense retrievers post overall CSCR of $0.32$-$0.40$, all \emph{worse} than our
stack's $0.287$; HyDE's LLM query expansion helps marginally on the average
($0.253$) and even posts the highest Recall@10 of any method ($0.448$), yet still
leaves SciFact at $0.56$ and NFCorpus at $0.50$. This confirms on the retriever
side what S1 showed on the reranker side: contamination is not an artifact of one
encoder and does not shrink with ranking quality - the best-ranking method here is
still badly contaminated on the domains that matter. The comparison is not
like-for-like - our stack adds sparse fusion and a reranker the pure-dense
baselines lack - but that is exactly the point: contamination persists across
pipeline strength and retriever family, and none of these systems controls it.
Crucially, none of these methods emits a per-domain contamination guarantee - they are exactly the ``infer
a domain but certify nothing'' systems of \Cref{sec:intro} - whereas C3R attaches a
certified budget to any of them.

\paragraph{S3 / S4: vacuity frontier and sample complexity (\Cref{fig:frontier},
\Cref{fig:nmin}).}
The certificate width as a function of router error $(\rhob,\pib)$ traces a
\emph{vacuity frontier} (\Cref{fig:frontier}): the document-trained probe
($\rho{=}0.45$) lands in the vacuous region, which is exactly why it fails its
gate \emph{visibly}; the query-style probe ($\rho{=}0.13$) lands in the usable
region. This is the honesty property of \Cref{sec:method} made quantitative.
\Cref{fig:nmin} validates the sample-complexity rule: at $\alpha{=}0.5$ the
theoretical onset is $n_{\min}\approx 12$ calibration queries per group, and
empirically feasibility turns on between $n_2{=}12$ and $16$ and is complete by
$24$, with \emph{zero} empirical violations wherever a certificate is issued.

\paragraph{Ablations (\Cref{tab:ablation}).}
Two design choices are load-bearing. Replacing the Hoeffding-Bentkus bound with
the plain Hoeffding margin makes even a feasible domain uncertifiable (its
certificate widens to $1.0$). Dropping the $\rhob/\pib$ correction produces a
\emph{tighter} reported certificate by ignoring router error - an unsafe claim,
and the failure mode the simulation exhibits. The \textsc{ambig} bucket and the
Bonferroni-vs-\v{S}id\'ak allocation are second-order.

\paragraph{Cost (\Cref{tab:runtime}).}
The control layer adds $0.007$\,ms/query of arithmetic (a mismatch score over
$100$ candidates plus a thresholded demotion) on top of the frozen stack (single-threaded CPU, NVIDIA GB10 host), and
all calibration is offline; the certified layer is effectively free at serving
time.

%% file: sec/floats.tex
\begin{figure}[t]\centering
  \begin{subfigure}{0.49\columnwidth}\includegraphics[width=\linewidth]{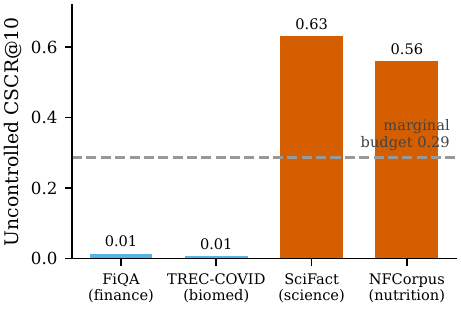}\caption{}\label{fig:asym}\end{subfigure}\hfill
  \begin{subfigure}{0.49\columnwidth}\includegraphics[width=\linewidth]{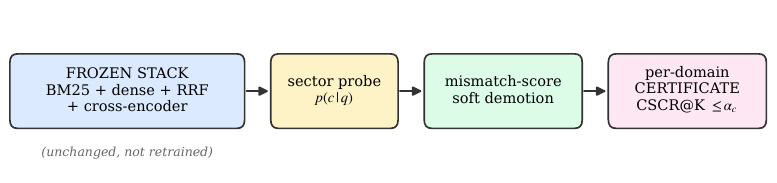}\caption{}\label{fig:pipeline}\end{subfigure}
  \caption{(a) Cross-domain contamination is sharply asymmetric: SciFact (0.63) and NFCorpus (0.56) exceed a marginal budget (0.29, dashed) that FiQA (0.01) satisfies, so a marginal guarantee cannot protect the contaminated domains. (b) C3R is a drop-in layer over a \emph{frozen} stack: an inferred domain posterior drives mismatch-score soft demotion and emits a per-domain certificate, with no query-time label.}
  \label{fig:hero}
\end{figure}
\begin{figure}[t]\centering\includegraphics[width=0.92\columnwidth]{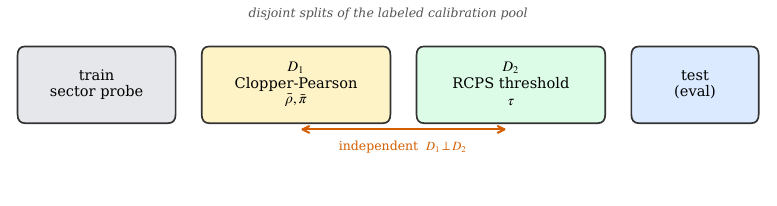}
\caption{Two-split calibration: $\Done$ bounds the router errors $\rhob,\pib$; the independent $\Dtwo$ sets the demotion threshold $\tau$. Independence ($\Done\perp\Dtwo$) makes the transfer bound non-circular.}\label{fig:splits}\end{figure}
\begin{figure}[t]\centering\includegraphics[width=0.86\columnwidth]{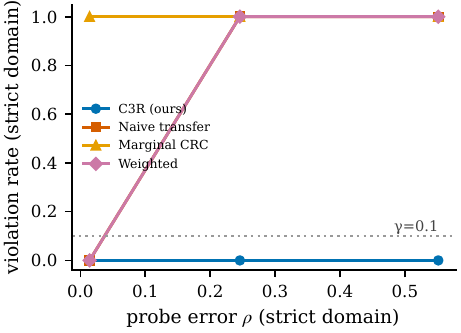}
\caption{Simulation: C3R's strict-domain violation rate stays at $0$ across router quality, while naive transfer, marginal control, and the weighted variant breach the nominal $\gamma{=}0.1$ as the router degrades.}\label{fig:sim}\end{figure}
\begin{figure}[t]\centering\includegraphics[width=\columnwidth]{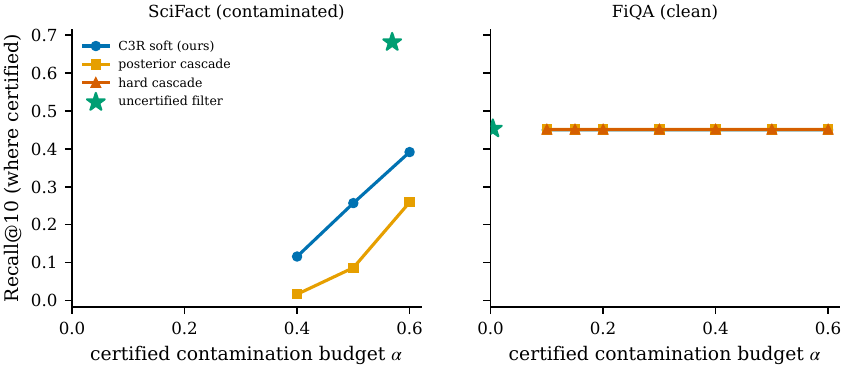}
\caption{Recall@10 retained \emph{where the budget is certified}, vs.\ contamination budget $\alpha$. On the contaminated domain (SciFact) C3R retains $2$-$6\times$ more than the strongest calibrated cascade and the hard cascade cannot certify; on the clean domain (FiQA) all methods tie. The uncertified classifier filter ($\star$) sits at $0.57$ contamination with no guarantee.}\label{fig:pareto}\end{figure}
\begin{figure}[t]\centering\includegraphics[width=0.82\columnwidth]{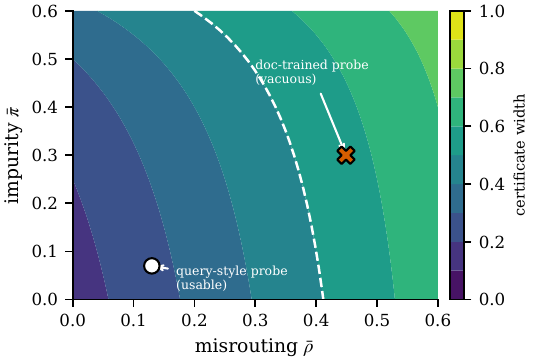}
\caption{Vacuity frontier: certificate width as a function of router error $(\rhob,\pib)$. A weak probe announces itself with a wide (visible) certificate rather than a silent violation; the doc-trained probe ($\times$) is vacuous, the query-style probe ($\circ$) usable.}\label{fig:frontier}\end{figure}
\begin{figure}[t]\centering\includegraphics[width=0.82\columnwidth]{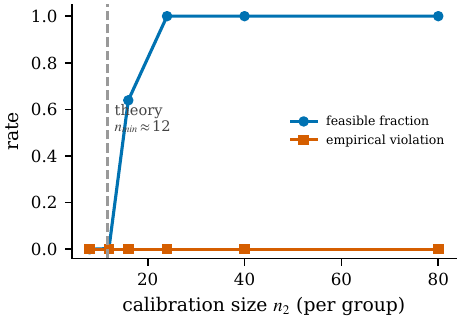}
\caption{Empirical sample complexity: feasibility onset matches the theoretical $n_{\min}\approx\ln(e/\delta)/\alpha\;(\approx 12)$, with zero empirical violations wherever a certificate is issued.}\label{fig:nmin}\end{figure}
\begin{table}[t]\centering\small
\setlength{\tabcolsep}{4.5pt}
\caption{Contamination is retriever-agnostic and uncertified. Five released dense retrievers, run as pure top-$10$ retrieval on the same Mix-A pool, all contaminate the biomedical domains heavily - independent of ranking quality (E5 nearly matches our Recall@10 yet contaminates at $0.33$; TAS-B contaminates most). None emits any contamination guarantee. C3R is a certified layer over \emph{any} such retriever. ``Ours'' is the full hybrid+rerank stack (uncontrolled); $\downarrow$/$\uparrow$ mark better.}
\label{tab:baselines}
\begin{tabular}{@{}lcccc@{}}
\toprule
Retriever & R@10\,$\uparrow$ & CSCR@10\,$\downarrow$ & \multicolumn{2}{c}{worst-domain CSCR\,$\downarrow$} \\
\cmidrule(l){4-5}
 & & (overall) & SciFact & NFCorpus \\
\midrule
BGE+rerank \emph{(ours)} & $0.417$ & $0.287$ & $0.63$ & $0.56$ \\
Contriever & $0.332$ & $0.321$ & $0.62$ & $0.72$ \\
GTR-base   & $0.363$ & $0.316$ & $0.67$ & $0.64$ \\
E5-base    & $0.414$ & $0.326$ & $0.67$ & $0.70$ \\
TAS-B      & $0.243$ & $0.397$ & $0.81$ & $0.84$ \\
all-MiniLM & $0.371$ & $0.326$ & $0.63$ & $0.72$ \\
HyDE \emph{(Qwen2.5-7B)} & \HyR & \HyC & \HySci & \HyNf \\
\midrule
\textbf{+ C3R certified} & \multicolumn{4}{c}{per-domain CSCR $\le\alpha_c$ certified (\Cref{tab:validity}, \Cref{fig:pareto})} \\
\bottomrule
\end{tabular}
\end{table}
\begin{table}[t]\centering\small
\setlength{\tabcolsep}{4pt}
\caption{Generalization across a contamination spectrum, including a
\emph{regulated-domain} testbed. Contamination tracks inter-domain subject
overlap, not topical adjacency. C3R's per-domain certificate is valid (zero
violations over $1000$ resampled calibrations) at \emph{every} point, and soft
demotion dominates the strongest calibrated cascade wherever a domain is
contaminated. All pools are open corpora (BEIR / CQADupStack / public US
\emph{Code of Federal Regulations}); contamination is measured against true
provenance. CSCR ranges rounded to two decimals. The last column reports soft
vs.\ cascade Recall@10 retention at a tight \emph{feasible} budget on a
contaminated domain; for \textsc{Sector-Bench} this is the tightest feasible
sector (environmental), not the highest-CSCR sector (Title~18), which is too
contaminated to certify at the swept budgets (\Cref{sec:exp}).}
\label{tab:generalize}
\begin{tabular}{@{}llccc@{}}
\toprule
Testbed & Domain pool & CSCR@10 & C3R cert. & soft vs.\ cascade \\
 & (4 domains) & range & viol./1000 & (tight feasible dom.) \\
\midrule
Mix-A \emph{(main)} & finance, biomed$\times$3 & $0.01$-$0.63$ & $0/1000$ & $0.39$ vs.\ $0.26$ \\
Mix-D \emph{(ctrl.\ high)} & science$\times$3, finance & $0.01$-$0.37$ & $0/1000$ & $0.39$ vs.\ $0.04$ \\
Mix-C \emph{(ctrl.\ low)} & tech Q\&A forums & $0.07$-$0.08$ & $0/1000$ & - \emph{(clean)} \\
\textsc{Sector-Bench} & CFR Titles 10/18/40/21 & $0.14$-$0.36$ & $0/1000$ & $0.58$ vs.\ $0.10$ \\
\bottomrule
\end{tabular}
\end{table}
\begin{table}[t]\centering\small
\setlength{\tabcolsep}{4pt}
\caption{The four open testbeds (all assembled from public corpora; domain $=$
source). Each query carries its source dataset's own relevance judgments (BEIR /
CQADupStack) or, for \textsc{Sector-Bench}, its section heading; a document's
domain is its objective source provenance, making contamination a checkable
phenomenon rather than a model artifact. Assembly scripts are released.}
\label{tab:testbeds}
\begin{tabular}{@{}lllrr@{}}
\toprule
Testbed & Source & Domains & \#docs & \#queries \\
\midrule
BEIR-MIX (Mix-A) & BEIR & finance, biomed$\times$3 & $237{,}786$ & $1{,}321$ \\
Mix-D & BEIR & science$\times$3, finance & $92{,}111$ & $2{,}271$ \\
Mix-C & CQADupStack & tech Q\&A$\times$4 & $165{,}080$ & $4{,}948$ \\
\textsc{Sector-Bench} & eCFR & energy/power/env./drug & $19{,}174$ & $5{,}744$ \\
\bottomrule
\end{tabular}
\end{table}

\begin{table}[t]\centering\small
\setlength{\tabcolsep}{5pt}
\caption{Feasibility and abstention per domain, at the locked target budgets. The
fraction of resampled calibrations in which the domain is \emph{feasible} (a
certificate tighter than $B{=}1$ is issued); the complement is the abstention rate.
Feasibility tracks the router floor $\rhob_\dom B$: clean-router domains certify on
every resample, moderate ones certify on a fraction, and domains whose floor
exceeds the requested budget abstain throughout. This is the quantity a deployer
reads to decide whether a domain is servable at a given budget.}
\label{tab:feasibility}
\begin{tabular}{@{}llcc@{}}
\toprule
Testbed & Domain & budget $\alpha^\star$ & feasible (non-abstain) \\
\midrule
Mix-A & FiQA & $0.15$ & $100\%$ \\
Mix-A & NFCorpus & $0.15$ & $34\%$ \\
Mix-A & SciFact & $0.05$ & $0\%$ (abstains) \\
Mix-A & TREC-COVID & $0.15$ & $0\%$ (abstains, small sample) \\
\midrule
Mix-D & SciDocs & $0.15$ & $100\%$ \\
Mix-D & FiQA & $0.15$ & $100\%$ \\
Mix-D & NFCorpus & $0.15$ & $56\%$ \\
Mix-D & SciFact & $0.05$ & $0\%$ (abstains) \\
\midrule
\textsc{Sector-Bench} & EPA/environ (Title 40) & $0.40$ & $100\%$ \\
\textsc{Sector-Bench} & FDA/drug (Title 21) & $0.40$ & $100\%$ \\
\textsc{Sector-Bench} & Energy/NRC (Title 10) & $0.40$ & $70\%$ \\
\textsc{Sector-Bench} & FERC/power (Title 18) & $0.40$ & $0\%$ (abstains) \\
\bottomrule
\end{tabular}
\end{table}

\begin{table}[t]\centering\small
\setlength{\tabcolsep}{4pt}
\caption{\textsc{Sector-Bench} per-sector detail (public eCFR; domain $=$ CFR
Title). Uncontrolled top-$10$ contamination (CSCR), feasibility at budget
$\alpha{=}0.4$, and soft-demotion vs.\ strongest-calibrated-cascade Recall@10
retention at \emph{equal certified contamination}. On the two cleaner, fully
feasible sectors soft demotion retains $4$ to $6\times$ the cascade; Energy/NRC is
feasible in $70\%$ of resamples and operates near its feasibility limit, where the
soft curve is steep; FERC/power, the most contaminated sector, abstains at this
budget. Resampling half-widths are $\le 0.01$ (\Cref{sec:exp}).}
\label{tab:sectordetail}
\begin{tabular}{@{}lrrccc@{}}
\toprule
Sector (Title) & \#docs & \#queries & CSCR & feasible & soft vs.\ cascade R@10 \\
\midrule
EPA/environ (40) & $6{,}000$ & $1{,}500$ & $0.14$ & $100\%$ & $0.58$ vs.\ $0.10$ \\
FDA/drug (21)    & $6{,}000$ & $1{,}500$ & $0.15$ & $100\%$ & $0.54$ vs.\ $0.16$ \\
Energy/NRC (10)  & $5{,}159$ & $1{,}500$ & $0.21$ & $70\%$  & $0.06$ vs.\ $0.21$ \\
FERC/power (18)  & $2{,}015$ & $1{,}244$ & $0.36$ & $0\%$   & n/a (abstains) \\
\bottomrule
\end{tabular}
\end{table}

%% file: figures/table1_validity.tex
\begin{table}[t]\centering\small
\setlength{\tabcolsep}{3.5pt}
\caption{Per-domain certificate, fully unpacked (BEIR-MIX, $1000$ resampled
calibrations, $\gamma{=}0.1$, locked budgets). The router floor $\rhob B$ is the
irreducible misrouted-mass term: when it exceeds the target budget $\alpha^\star$
the rescaled run-budget is non-positive, the domain is \emph{infeasible}, and C3R
\emph{abstains} (issues the vacuous certificate $1.0$ and returns nothing) rather
than violate. The key contrast is the last two columns: C3R never violates (its
empirical CSCR stays $\le$ its issued certificate, including by abstaining),
whereas marginal CRC commits to a single budget and violates it on the
contaminated domains. ``-'' marks an abstaining domain.}
\label{tab:validity}
\resizebox{\columnwidth}{!}{%
\begin{tabular}{lccccccccc}
\toprule
Domain & $\alpha^\star$ & $\rhob B$ & resc.\ $\alpha$ & feas. & C3R cert & C3R viol & marg.\ viol & CSCR & R@10 \\
\midrule
FiQA        & $0.15$ & $0.03$ & $0.12$ & $100\%$ & $\mathbf{0.15}$ & $0.00$ & $0.00$ & $0.015$ & $0.45$ \\
TREC-COVID  & $0.15$ & $0.28$ & -     & $0\%$   & $1.0^{\dagger}$ & $0.00$ & $0.00$ & -      & -     \\
SciFact     & $0.05$ & $0.23$ & -     & $0\%$   & $1.0^{\dagger}$ & $0.00$ & $\mathbf{1.00}$ & - & -     \\
NFCorpus    & $0.15$ & $0.11$ & $0.04$ & $34\%$  & $0.71$          & $0.00$ & $\mathbf{0.53}$ & -      & -     \\
\bottomrule
\end{tabular}}

\smallskip
{\footnotesize $^{\dagger}$ abstains: the router floor $\rhob B$ exceeds the
requested budget, so no certificate tighter than $B{=}1$ is issuable. The two
abstentions have different causes: SciFact's is a genuine router-error floor,
whereas TREC-COVID's $\rhob B{=}0.28$ is inflated by its small sample ($50$ test
queries, hence a loose Clopper-Pearson bound), not a contamination floor. The
oracle that uses test-time labels also never violates (omitted for space).
Recall-bearing certified operating points are the feasible budgets of
\Cref{fig:pareto}.}
\end{table}

%% file: figures/table2_wrap.tex
\begin{table}[t]\centering
\caption{A stronger reranker improves ranking but does \emph{not} reduce contamination; C3R's certificate is reranker-agnostic.}
\label{tab:wrap}
\begin{tabular}{lccc}
\toprule
Reranker & Recall@10 & nDCG@10 & overall CSCR@10 \\
\midrule
MiniLM-L6 & 0.417 & 0.406 & 0.287 \\
Qwen3-Reranker-8B & 0.522 & 0.522 & 0.258 \\
\bottomrule
\end{tabular}\end{table}

%% file: figures/table3_priorart.tex
\begin{table*}[t]\centering
\caption{Positioning vs prior work. No prior method combines risk control of a bounded loss, a \emph{latent} test-time group, and heterogeneous per-group budgets.}
\label{tab:priorart}
\resizebox{\textwidth}{!}{%
\begin{tabular}{llcccc}
\toprule
Work & Guarantee type & Group at test time & Risk vs coverage & Per-group budgets & Missing for our setting \\
\midrule
Sesia et al.\ \cite{sesia2024noisy} & noise-corrected coverage & class (standard) & coverage & no & latent group; risk control \\
Einbinder et al.\ \cite{einbinder2024labelnoise} & robust coverage & none & coverage & no & groups; risk; budgets \\
Gibbs-Cherian-Cand\`es \cite{gibbs2025conditional} & conditional coverage & computable from $x$ & coverage & no & latent group; risk control \\
Kandinsky CP \cite{bairaktari2025kandinsky} & weighted cond.\ coverage & fractional (estimated) & coverage & no & risk control; true-group transfer \\
Posterior CP \cite{zhang2024posterior} & approx.\ cond.\ validity & data-driven clusters & coverage & no & risk; quantified slack; budgets \\
C-RAG \cite{kang2024crag} & conformal risk control & none & risk (generation) & no & groups; per-group budgets \\
CCRA-S \cite{ccra2025} & group risk control & observed strata & risk & per-stratum & \emph{latent} group; transfer \\
\midrule
\textbf{C3R (ours)} & \textbf{conformal risk control} & \textbf{latent (inferred)} & \textbf{risk (contamination)} & \textbf{heterogeneous} & \textbf{ - } \\
\bottomrule
\end{tabular}}\end{table*}

%% file: figures/table4_ablation.tex
\begin{table}[t]\centering
\caption{Ablation: NFCorpus certificate width. Hoeffding cannot certify; dropping the $\bar\rho/\bar\pi$ correction yields an unsafe (too-tight) certificate.}
\label{tab:ablation}
\begin{tabular}{lc}
\toprule
Variant & NFCorpus cert. \\
\midrule
Full (HB, Bonferroni, $\bar\rho/\bar\pi$, AMBIG) & 0.709 \\
$-$ $\bar\rho/\bar\pi$ correction (unsafe) & 0.150 \\
$-$ AMBIG bucket & 0.685 \\
Hoeffding (vs HB) & 1.000 \\
\v{S}id\'ak (vs Bonferroni) & 0.693 \\
\bottomrule
\end{tabular}\end{table}

%% file: figures/table5_runtime.tex
\begin{table}[t]\centering
\caption{The certified control layer adds negligible arithmetic atop the frozen retrieval stack; calibration is offline.}
\label{tab:runtime}
\begin{tabular}{lc}
\toprule
Component & Cost \\
\midrule
Control layer (mismatch + demotion) & 0.007 ms/query \\
Calibration ($D_1,D_2$) & offline \\
Frozen BM25+dense+rerank stack & dominates \\
\bottomrule
\end{tabular}\end{table}

%% file: sec/7_conclusion.tex
\section{Conclusion}
\label{sec:conclusion}

We presented C3R, a drop-in layer that certifies a per-domain contamination
budget for multi-domain retrieval using only an inferred domain posterior. Its
core is a two-split conformal scheme whose transfer bound crosses from the
inferred domain to the true domain with fully estimable slack, supports
heterogeneous budgets, and inverts for deployment. On an open multi-domain
testbed the certificate never violates while marginal control fails the
contaminated domains, and soft demotion retains markedly more recall than the
strongest calibrated hard filter where budgets are tight.

\paragraph{Deploying C3R in practice.}
The method exposes three knobs, and our experiments suggest concrete defaults.
\emph{The probe} need not be strong, only \emph{calibrated}: a temperature-scaled
logistic head on frozen embeddings places the query-style router in the usable
region of the vacuity frontier ($\rhob\approx0.13$), whereas a document-trained
probe is vacuous; a practitioner should check $\rhob$ on a held-out split before
trusting the certificate, since the certificate width is governed by it.
\emph{The budget} $\alpha_\dom$ should be set per domain from the tolerable
contamination of the \emph{consumer}, not the retriever: a strict $\alpha$ on a
domain whose router floor $\rhob B$ exceeds it will simply return nothing, which
is the safe action but an operational cost, so a deployer who needs availability
on a hard domain must either accept a looser certified budget (a guaranteed
\emph{reduction} rather than a tight bound) or invest in a better router.
\emph{Abstention} should be surfaced, not hidden: an empty certified result set is
a signal that the requested guarantee is infeasible at the current router quality,
and is more useful to a downstream compliance workflow than a silently
contaminated list. Because the layer is frozen-stack and adds $0.007$\,ms/query,
it can be A/B-tested behind an existing retriever with no retraining.

\paragraph{Is provenance the right object to certify?}
We control cross-domain \emph{provenance} mismatch (CSCR), and our appropriateness
probe shows this is not the same as \emph{harm}: only $17\%$ of cross-domain
retrievals are wrong-authority, while $54\%$ are benign cross-cutting regulation
that a reader could legitimately use. CSCR therefore over-counts harm, and
demoting on it has a recall cost - the visible price of acting on a conservative
signal. We argue provenance is nonetheless the right target for a \emph{label-free,
inference-time} guarantee, for two reasons. First, appropriateness is not
computable at query time without the very labels the setting denies us; provenance
is the strongest objective signal available without an oracle, and it
\emph{conservatively contains} the harmful subset (wrong-authority $\subseteq$
wrong-provenance for cross-domain errors), so a provenance certificate bounds that
harm from above. Second, soft demotion turns the recall cost into a tunable
trade-off rather than a hard filter: the budget $\alpha$ chooses how much benign
cross-cutting evidence to keep in exchange for a tighter harm bound. A method that
certified appropriateness \emph{directly} would need an appropriateness oracle at
calibration time; integrating even a partial human-annotated layer into the
certified loss - so the budget is on harm, not provenance - is the most valuable
extension this work points to.

\paragraph{Scope and limitations.}
Our BEIR domains are general-purpose corpora; \textsc{Sector-Bench}
(\Cref{tab:generalize}) closes part of the gap to the regulated setting by
certifying contamination on actual public federal regulation, where provenance is
objective. It does not directly certify \emph{appropriateness} - whether a wrong-provenance
passage carries genuinely wrong regulatory authority. An LLM-judged probe
(\Cref{sec:exp}) gives a first answer - $17\%$ of contaminating retrievals are
wrong-authority versus $0\%$ within-domain, and since these are a subset of
wrong-provenance ones the CSCR certificate bounds them - but a \emph{human}-annotated
appropriateness layer with inter-annotator agreement, and settings where domain
boundaries are contested or labels are truly hidden, remain the central pieces of
future work.
Two further caveats: TREC-COVID contributes only $50$ test queries, so its
per-domain certificate is a small-sample stress case that often abstains; and
the naive scheme's failure mode is established in simulation rather than on
BEIR-MIX, because our probe is good enough that naive transfer does not visibly
break on this benchmark. The sample-complexity rule $n_{\min}\approx
\ln(e/\delta)/\alpha$ is theoretically derived and empirically supported on the
main testbed; we replicate the certificate's validity and the soft-demotion
advantage on three further open pools spanning the contamination spectrum
(\Cref{tab:generalize}), but the subject-overlap trend rests on only four
constructed pools and merits a larger systematic study.

\paragraph{Threats to validity.}
We name the main ones explicitly. \emph{External validity:} all testbeds are
general-purpose public corpora with objectively known provenance; whether the
certificate transfers to regulated sectors, where domain boundaries are
contested and labels are costly, is unvalidated and is the central future-work
item. \emph{Probe sensitivity:} the certificate width is governed by the router
errors $\rhob,\pib$, which depend on the probe and its training split; a weaker
probe yields a wider (but still valid and \emph{visible}) certificate, and at the
extreme the bound is vacuous (\Cref{fig:frontier}) - so results are contingent on
a probe in the usable region, not on a perfect router. \emph{Dataset bias:}
domain $=$ source dataset, so the contamination we measure inherits each BEIR
dataset's topic and qrel conventions; the magnitude of contamination is therefore
a property of the pool composition (\Cref{tab:generalize}), not a universal
constant.

\paragraph{Future work.}
The natural next step is a regulated-domain benchmark with expert domain and
appropriateness annotations, on which the inferred-domain certificate can be
stress-tested where it matters most; online recalibration under distribution
shift; and joint certification of contamination together with relevance. We will
release the BEIR-MIX assembly, the CSCR evaluation protocol, and the C3R
calibration code to support these directions.

%% file: sec/8_ethics.tex
\section*{Ethical Considerations}
This work studies a reliability mechanism for retrieval and uses only public,
open corpora - BEIR datasets, CQADupStack forums, and public US CFR/eCFR
regulatory text - with released or objective (source-provenance) labels; no human
subjects, personal data, or proprietary corpora are involved. The intended
effect of C3R is to \emph{reduce} the risk that retrieval-augmented systems
surface wrong-domain evidence, which is a safety-positive goal in high-stakes
settings. We note two honest caveats. First, a certificate is only as meaningful
as its stated assumptions: deployers must not read a per-domain contamination
bound as a guarantee of factual correctness or of fitness for a regulated
purpose. \textsc{Sector-Bench} validates provenance-based contamination control on
public regulatory text, but not expert-judged legal or operational
appropriateness, which a certified provenance bound does not supply.
Second, the domain probe could encode dataset-specific biases; because the method
is label-free at inference, an over-trusted but miscalibrated probe could degrade
service for under-represented query types. C3R mitigates this by making router
error visible in the certificate width rather than silent, but we recommend
auditing probe calibration before deployment. We provide a publicly released artifact
(code, result files, benchmark assembly) to support independent scrutiny.

%% file: sec/A_appendix.tex
\section{Proof of the Per-Domain Transfer Certificate}
\label{app:proof}

We prove the per-domain bound \eqref{eq:cert}; simultaneity over domains follows
from the Bonferroni allocation and union bound discussed after the theorem.

Fix a domain $\dom$ and condition on $y(q)=\dom$. Write $L=\cscr@K(q)$ for the
served contamination loss and partition on the router outcome:
\begin{align}
\E[L\mid y=\dom]
&=\Pr(\ghat=\dom\mid y=\dom)\,\E[L\mid y=\dom,\ghat=\dom]\nonumber\\
&\quad+\Pr(\ghat\neq\dom\mid y=\dom)\,\E[L\mid y=\dom,\ghat\neq\dom]\nonumber\\
&\le (1-\rho_\dom)\,\E[L\mid y=\dom,\ghat=\dom]+\rho_\dom B,
\label{eq:appdecomp}
\end{align}
using $\Pr(\ghat\neq\dom\mid y=\dom)=\rho_\dom$ and $L\le B$
(\Cref{ass:bdd}) for the misrouted mass.

\paragraph{Correctly-routed term.}
Queries with $\ghat=\dom$ are served threshold $\tau_\dom$, and the
group-$\dom$ RCPS step certifies, with confidence $1-\delta_\dom$,
$\E[L\mid\ghat=\dom]\le\alpha_\dom$, a bound on the risk over the \emph{inferred}
group (the finite-sample correction is absorbed into the choice of $\tau_\dom$,
not added to $\alpha_\dom$). We need the risk over the truly-$\dom$ subset of that
group. Let $A=\{y=\dom\}$ and $G=\{\ghat=\dom\}$. By definition
$\pi_\dom=\Pr(\neg A\mid G)$, so $\Pr(A\mid G)=1-\pi_\dom$, and since $L\ge 0$,
\begin{equation}
\E[L\mid G]\ge \Pr(A\mid G)\,\E[L\mid G,A]=(1-\pi_\dom)\,\E[L\mid y=\dom,\ghat=\dom].
\end{equation}
Rearranging and applying the RCPS bound,
$\E[L\mid y=\dom,\ghat=\dom]\le \E[L\mid G]/(1-\pi_\dom)\le
\alpha_\dom/(1-\pi_\dom)$. Substituting into \eqref{eq:appdecomp},
\begin{equation}
\E[L\mid y=\dom]\le (1-\rho_\dom)\frac{\alpha_\dom}{1-\pi_\dom}+\rho_\dom B .
\label{eq:apptrue}
\end{equation}

\paragraph{Plugging in estimated bounds (non-circularity).}
\eqref{eq:apptrue} is in terms of the \emph{true} $\rho_\dom,\pi_\dom$, which we
do not know. The right-hand side is non-decreasing in $\rho_\dom$ whenever
$B\ge \alpha_\dom/(1-\pi_\dom)$ (the regime of interest, since $B$ upper-bounds
the loss) and non-decreasing in $\pi_\dom$; therefore replacing them with the
one-sided Clopper-Pearson \emph{upper} bounds $\rhob_\dom,\pib_\dom$ can only
increase it, preserving the inequality. The estimates $\rhob_\dom,\pib_\dom$ are
computed on $\Done$, while $\tau_\dom$ (hence the certified $\alpha_\dom$) is
computed on $\Dtwo$; by \Cref{ass:indep} these are independent, so the binomial
coverage events and the RCPS risk-control event are jointly controlled by a
union bound and no data is used both to fit and to certify the same threshold.
This is the step that makes the slack estimable rather than circular. Combining
the binomial coverage (each $\rhob,\pib$ valid with its allocated error) with
\eqref{eq:apptrue} yields \eqref{eq:cert}. \hfill$\qed$

\paragraph{Remark (impossibility scope).}
\Cref{thm:main} is an \emph{approximate} group-conditional statement: the slack
over the group budget - the router-error term $\rhob_\dom B$ and the impurity
inflation factor $1/(1-\pib_\dom)$ - does not vanish unless the router is perfect
($\rhob_\dom=\pib_\dom=0$). This is
unavoidable - exact conditional coverage with an estimated conditioning variable
is impossible in finite samples~\cite{barber2021limits} - and is the price of
operating without test-time labels. The contribution is that the slack is fully
estimable and degrades gracefully and visibly with router quality.

\section{Additional Details}
\label{app:details}

\paragraph{Calibration protocol.}
For each domain $\dom$ we draw, per resample, a disjoint partition of the
truly-$\dom$ queries: $30\%$ are held out for evaluation and the remaining $70\%$
are split evenly into $\Done$ (router-error estimation) and $\Dtwo$ (threshold
calibration); the domain probe is trained once on a third, disjoint pool of
queries and frozen. On $\Done$ we form one-sided Clopper-Pearson upper bounds
$\rhob_\dom,\pib_\dom$ at level $1-\gamma_{\mathrm{each}}$, where the global error
$\gamma{=}0.1$ is split across the $3C{+}1$ statements ($2C$ binomial bounds and
$C{+}1$ group risk-control tests) by a Bonferroni allocation, which is
unconditionally valid and is the allocation behind every reported certificate (a
\v{S}id\'ak allocation is second order, \Cref{tab:ablation}). On $\Dtwo$ we
run group-$\dom$ RCPS at the rescaled level $\alpha_\dom$ of \eqref{eq:rescale}
and read the largest feasible threshold $\tau_\dom$. All reported quantities are
means over $1000$ such resamples; the demotion threshold grid is
$\tau\in\{0,0.02,\dots,1\}$ ($51$ points), $K{=}10$, and the bounded-loss range is
$B{=}1$.

\paragraph{Hoeffding-Bentkus RCPS.}
The Hoeffding-Bentkus RCPS $p$-value for a group of $n$ losses with empirical
mean $\hat L$ against level $\alpha$ is $\min\{\,e\cdot\Pr(\mathrm{Bin}(n,\alpha)
\le \lceil n\hat L\rceil),\ \exp(-2n(\alpha-\hat L)_+^2)\,\}$~\cite{bates2021rcps};
we certify a threshold $\tau$ when this is $\le\delta$, taking the most permissive
$\tau$ that remains certified. The Bentkus term dominates (is tighter) when the
observed loss sits well below the budget, which is the regime our calibration
operates in; the additive Hoeffding margin alone (the ablation of
\Cref{tab:ablation}) is far looser and renders even feasible domains
uncertifiable.

\paragraph{Budget-rescaling inversion.}
The corollary \eqref{eq:rescale} is the exact inverse of the certificate
\eqref{eq:cert}. Setting the right-hand side of \eqref{eq:cert} equal to a target
$\alpha^\star_\dom$ and solving for the run-budget $\alpha_\dom$,
\begin{equation*}
\alpha^\star_\dom=(1-\rhob_\dom)\frac{\alpha_\dom}{1-\pib_\dom}+\rhob_\dom B
\;\Longleftrightarrow\;
\alpha_\dom=\big(\alpha^\star_\dom-\rhob_\dom B\big)\frac{1-\pib_\dom}{1-\rhob_\dom},
\end{equation*}
which is \eqref{eq:rescale}. The map is monotone increasing in $\alpha^\star_\dom$,
so a tighter target yields a tighter run-budget; it crosses zero at
$\alpha^\star_\dom=\rhob_\dom B$, the router floor, below which $\alpha_\dom\le 0$
and the domain is infeasible. Running group-$\dom$ RCPS at this $\alpha_\dom$ and
substituting back reproduces $\alpha^\star_\dom$ on the left of \eqref{eq:cert},
so the delivered per-domain certificate equals the requested target exactly when
feasible.

\paragraph{Sample-complexity onset.}
The feasibility threshold $n_{\min}$ follows from the Bentkus term. With observed
group loss $\hat L\to 0$ (an aggressively demoted, clean served set), the
Hoeffding-Bentkus $p$-value reduces to $e\cdot\Pr(\mathrm{Bin}(n,\alpha){=}0)
=e\,(1-\alpha)^n$. Certification requires this to be $\le\delta$, i.e.
$(1-\alpha)^n\le\delta/e$, hence
\begin{equation*}
n\;\ge\;\frac{\ln(e/\delta)}{-\ln(1-\alpha)}\;\approx\;\frac{\ln(e/\delta)}{\alpha}
\quad(\text{small }\alpha),
\end{equation*}
the $n_{\min}\approx\ln(e/\delta)/\alpha$ used in \Cref{sec:theory}. Since
$-\ln(1-\alpha)\ge\alpha$, the approximation $\ln(e/\delta)/\alpha$ slightly
\emph{over}estimates the exact threshold $\ln(e/\delta)/(-\ln(1-\alpha))$, so the
stated $n_{\min}$ is a mildly \emph{conservative} leading-order estimate and the
true onset is marginally smaller; \Cref{fig:nmin} confirms the empirical
feasibility turns on at the predicted scale. A domain is therefore servable at run-budget $\alpha_\dom$ once
its $\Dtwo$ group exceeds $\approx\ln(e/\delta)/\alpha_\dom$ queries, which, with
the rescaling above, is the precise sense in which a stricter target or a weaker
router (larger $\rhob_\dom B$, smaller $\alpha_\dom$) demands more calibration data
or forces abstention.

\paragraph{Reproducibility.}
The frozen retrieval stack is BM25 ($k_1{=}0.9,b{=}0.4$) $+$ BGE-base dense with
RRF ($k{=}60$) and a cross-encoder reranker trained on MS-MARCO~\cite{nguyen2016msmarco}; the probe is a
temperature-scaled logistic head (an MLP variant in the ablations); the random
seed is $42$ throughout. The probe, per-domain splits, locked budgets, every
result file behind each figure and table, and the assembly scripts for all four
testbeds (\Cref{tab:testbeds}) are provided in the publicly released artifact.